\documentclass[letterpaper]{article} 
\usepackage[preprint]{aaai2027}
\usepackage[hyphens]{url}  
\usepackage{graphicx} 
\usepackage{natbib}  
\usepackage{caption} 
\usepackage{booktabs}
\usepackage{amsmath}
\usepackage{amssymb}
\usepackage{amsthm}
\newtheorem{proposition}{Proposition}
\newtheorem{lemma}{Lemma}
\newtheorem{theorem}{Theorem}
\title{A Controlled Study of Attention-Only Transformers}
\author{
    Henry Ndubuaku, Karen Mosoyan, Jakub Mroz, Noah Cylich,\\
    Satyajit Kumar, Parkirat Sandhu, Roman Shemet, Justin H Lee
}
\affiliations{
    Cactus Compute, Inc.
}

\begin{document}

\maketitle

\begin{abstract}
Feed-forward networks hold two thirds of a transformer's non-embedding
parameters, yet the architecture has not received a necessity test that
controls parameters, compute, and depth at once. We pretrain attention-only
decoder transformers (Simple Attention Networks, SANs) against standard
transformers matched separately for parameter count, training FLOPs, and
depth (2 to 48 layers across arms), with per-arm learning-rate sweeps, for
up to 105B tokens (1.5 epochs of a 68B-token reasoning-dense corpus) at 6M
to 87M total parameters. Deleting feed-forward layers in place is costly: the
standard transformer leads by 0.47 nats at matched depth, and by 0.26 nats
at matched training FLOPs, where attention spends compute on the
parameter-free quadratic term and so carries fewer parameters at equal
cost. Reallocating the freed budget into attention depth closes the gap:
at matched parameter count the difference is 0.006 nats, 0.27 percent of
loss, reproducible to one part in ten thousand across clean seed pairs,
shrinking across separately trained 5B, 30B, and 105B budgets, and holding
near 0.02 nats across a 29x non-embedding size range at a fixed
31.5B-token budget. Three independent
measurements localize the remaining gap to parametric recall: loss over
token regions, loss over task types, and zero-shot benchmarks.
Attention-only models are better on context-grounded answers and worse
where knowledge must come from weights. Weight spectra show why: routing
matrices (Q/K) crystallize in the first quarter of training while content
matrices accumulate rank through the stable phase, and removing
feed-forward layers relocates this accumulation to the attention output
projection.
QK-normalization, not feed-forward layers or residual gating, keeps
48-layer attention-only stacks trainable. The gap is concentrated, not
diffuse: low-context query prediction carries a per-token deficit five
times the aggregate but only 8 percent of corpus loss, and the
token-weighted decomposition reproduces its own aggregate to within 2
percent; by the largest budget the deficit localizes there entirely on
that sample. A
pre-registered test confirms the account: it predicts a 0.02 to 0.05 nat
gap on knowledge-dense web text, and a matched pair trained on fineweb-edu
measures 0.040. Within the tested regime, attention does the rest.
\end{abstract}

\section{Introduction}

The feed-forward network (FFN) has been a fixed component of the
transformer block since its introduction \citep{vaswani2017attention}, and
in modern decoder configurations it holds roughly two thirds of
non-embedding parameters. A substantial interpretability literature argues
that these layers act as the model's parametric memory: key--value stores
over training facts \citep{geva2021transformer,geva2022transformer}, the
locus of editable factual associations \citep{meng2022locating}, and the
home of identifiable knowledge neurons \citep{dai2022knowledge}. Yet the
converse experiment, removing the FFN entirely and measuring what is
actually lost, has never been run under controls. The closest attempt
folded equivalent capacity back into attention as persistent memory
vectors \citep{sukhbaatar2019augmenting}; probing and editing studies
localize function without deleting it; and the classical theory of pure
attention \citep{dong2021attention} analyzes a setting without residual
connections that no one trains.

The experiment is harder than it looks because removing the FFN perturbs
three quantities at once: parameter count, compute per token, and the
number of nonlinear composition steps. A single-control comparison is
therefore confounded no matter which axis it fixes. We run the necessity
test with three separate matchings, each holding one axis fixed
(Table~\ref{tab:configs}), give every arm its own learning-rate sweep with
a boundary-extension protocol so that no arm competes at another arm's
preferred hyperparameters, and calibrate the experiment's noise floor with
an optimizer-equivalent control pair. All predictions were registered
before the corresponding measurements were made, and one, the gap on
knowledge-dense text, was registered numerically before the run was
launched.

The result is a monotone sequence that decomposes the question. Deleting
feed-forward layers in place, at matched depth, costs 0.47 nats of
validation loss. At matched training FLOPs, where attention's quadratic
term crowds out parameters, the standard transformer still leads by 0.26
nats. But at matched parameter count, with the freed budget reallocated
into attention depth, the gap collapses to a small, highly reproducible
$+0.0055$ and $+0.0054$ nats on clean seed pairs (0.27\% of loss, agreeing
to $10^{-4}$), and it shrinks monotonically across separately trained 5B,
30B, and 105B token budgets while staying near 0.02 nats across a
29$\times$ non-embedding parameter range at a fixed 31.5B-token budget.
The FFN's parameters matter; its functional form largely does not, on
this distribution.

What remains of the gap is not diffuse. Decomposing the loss over token
regions and task types, and separately over zero-shot benchmarks, the
deficit concentrates on low-context prediction: tokens for which the
context offers little to route, so that only parametric knowledge helps.
At the largest budget the attention-only model is ahead on every answer
region of the decomposition sample, including memorization exercises, and
behind only on low-context query tokens; on benchmarks, the one task whose
answer sits in
a provided passage favors the SAN with growing margin, while the one task
demanding out-of-distribution recall favors the FFN model at every budget.
A weight-spectrum analysis supplies the mechanism: in every model, routing
matrices (Q/K) spectrally crystallize within the first quarter of
training, while content matrices accumulate stable rank through the
stable phase; removing the FFN relocates this accumulation to the
attention output projection, the attention-only model's only write path
into the residual stream.

Our contributions: (i) the first controlled necessity test of the
transformer FFN, with parameter, FLOP, and depth matchings, per-arm
learning-rate fairness, and a calibrated noise floor; (ii) a three-way
localization of the residual gap to low-context prediction, confirmed by a
numerically pre-registered out-of-distribution test (predicted 0.02--0.05
nats on fineweb-edu; measured 0.040); (iii) the identification of
QK-normalization, not feed-forward layers or residual gating, as the
component that keeps deep attention-only stacks trainable; and (iv) a
weight-spectrum account in which routing crystallizes early, content
accumulates through the stable phase, and storage relocates when the FFN
is removed.

\section{Related Work}

\textbf{The FFN as memory.} \citet{geva2021transformer} identify FFN
layers as key--value memories; \citet{geva2022transformer} refine the
mechanism; ROME \citep{meng2022locating} and knowledge neurons
\citep{dai2022knowledge} localize editable facts in mid-layer FFNs. These
are probing and editing results; ours is the causal complement, showing by
deletion that the capability lost is precisely parametric storage.
Attention weights store associations too \citep{elhage2021mathematical},
so storage is a tendency with causal evidence rather than an exclusive
location; our decomposition measures the tendency's size instead of
assuming it.
\citet{pires2023one} find FFNs redundant across layers;
\citet{he2024simplifying} remove other block components via
signal-propagation arguments.

\textbf{Attention-only models.} \citet{sukhbaatar2019augmenting} remove
the FFN but restore its capacity as persistent memory vectors, leaving the
necessity question open. The mechanistic-interpretability literature
studies small attention-only transformers as an analyzable model class
\citep{elhage2021mathematical,olsson2022context}, establishing the QK/OV
routing-versus-content decomposition that our spectral dynamics confirm at
realistic depth. \citet{dong2021attention} prove rank collapse for
attention stacks without residuals; our 48-layer results and rank
measurements bound where that analysis applies to trained practice, and
identify QK-normalization \citep{henry2020query,dehghani2023scaling} as
the operative stabilizer.

\textbf{The dual ablation.} MLP-only architectures
\citep{tolstikhin2021mlp,liu2021pay} map the attention-free direction;
MetaFormer \citep{yu2022metaformer} argues the block frame matters more
than the mixer. We supply the missing attention-only side for language
modeling: deleting the channel MLP costs 0.006 nats at matched parameters,
with a specific, predictable residual profile.

\textbf{Components and optimizer.} Our scalar residual gates belong to the
Highway/ReZero/LayerScale family
\citep{srivastava2015highway,bachlechner2020rezero}; sandwich
normalization follows \citet{shleifer2021normformer}; zero-centered norm
gains follow \citet{olmo2024olmo2}; gated attention outputs are studied by
\citet{qiu2025gated}, and depth-wise residual reweighting by
\citet{kimi2026attnres}. We train with Muon
\citep{jordan2024muon,liu2025muon} against an AdamW control.

\textbf{Data regime.} We pretrain on SYNTH \citep{pleias2025synth}, a
reasoning-dense synthetic corpus, deliberately overtrained relative to
compute-optimal \citep{hoffmann2022chinchilla}, in the small-model
synthetic-pretraining lineage of
\citet{eldan2023tinystories,gunasekar2023textbooks}; repetition behavior
extends \citet{muennighoff2023scaling} to synthetic data. Chain-of-thought
expressivity results \citep{merrill2024expressive,feng2023towards}
motivate the trace-formatted regime, with the caveat that those
constructions use MLPs.

\section{Experimental Setup}

\textbf{Architecture.} A SAN block is a pre-norm attention block with the
FFN deleted, exactly:
\begin{align*}
\mathrm{ZCN}(z) &= (1+\gamma)\odot z/\mathrm{RMS}(z), \quad \gamma\
\text{init } 0\\
u &= \mathrm{ZCN}(x)\\
q,k,v &= W_q u,\ W_k u,\ W_v u \quad \text{(GQA: 8Q, 4KV)}\\
q,k &= \mathrm{RoPE}\big(\mathrm{ZCN}_h(q),\ \mathrm{ZCN}_h(k)\big)\\
A &= \mathrm{softmax}\big(qk^{\top}/\sqrt{d_h} + M\big)\\
y &= x + \sigma(g)\cdot W_o(A\,v)
\end{align*}
with $M$ the causal, document-boundary mask over packed 2048-token rows,
$g$ a scalar per sublayer (init 0, so every branch starts at half
strength), tied embeddings, and a final ZCN before the head. The control
arm inserts a SwiGLU FFN, $y' = y + \sigma(g_2)\cdot
W_{\mathrm{down}}\,\mathrm{SwiGLU}(W_{\mathrm{gate}}\tilde u,
W_{\mathrm{up}}\tilde u)$ with $d_{ff}=4d$, after attention; both arms
share one implementation. The complete inventory of per-position nonlinearities in
the SAN is the attention softmax and the normalizations: there is no
learned per-position feature map, which is the entire experimental
manipulation. Four short results make precise what the manipulation
removes.

\begin{proposition}[Conditional linearity]
For any fixed row-stochastic attention pattern $A$, the map from the
normalized context $(\hat u_1,\dots,\hat u_T)$ to the layer's updates is
linear.
\end{proposition}
\begin{proof}
It is the composition of the linear maps $W_v$, $v \mapsto Av$, and
$W_o$, scaled by the constant $\sigma(g)$.
\end{proof}

\begin{proposition}[Simplex transport]
Per head, the pre-projection update at position $i$ lies in the convex
hull of $\{W_v \hat u_j : j \le i\}$.
\end{proposition}
\begin{proof}
Softmax rows are nonnegative and sum to one.
\end{proof}

\noindent A SAN layer selects and transports content present in context;
it cannot synthesize representations unsupported by it. This is the formal
sense of ``context-grounded'' used throughout.

\begin{lemma}[Scalar bottleneck]
At sequence length 1, the entire depth-$L$ network is the linear map
$x_L = \big[\prod_\ell (I + \tilde B_\ell/\rho_\ell)\big]x_0$, where
$\tilde B_\ell = \sigma(g_\ell) W_o W_v\,\mathrm{diag}(1+\gamma_\ell)$ is
constant and $\rho_\ell = \mathrm{RMS}(x_\ell)$: a linear map modulated by
exactly $L$ scalars.
\end{lemma}
\begin{proof}
$A = [1]$ and RoPE at position 0 is the identity, so each block computes
$x_{\ell+1} = (I + \tilde B_\ell/\rho_\ell)\,x_\ell$; induct over layers.
\end{proof}

\noindent All per-position nonlinearity flows through this $L$-scalar
bottleneck; cross-token attention is the only escape, and pricing that
restriction is what the experiments do.

\begin{theorem}[Bounded stream variance at init]
Under the implemented initialization (exact RMS normalization, gates
$\sigma(0)=\tfrac12$, output projections i.i.d.\ mean-zero with variance
$s^2/(2N)$), $\mathbb{E}\|x_N\|^2 = \mathbb{E}\|x_0\|^2 + \Theta(1)$
uniformly in depth $N$.
\end{theorem}
\begin{proof}
Cross terms $\mathbb{E}\langle x_\ell, b_\ell\rangle$ vanish because
$W_o$ is mean-zero and independent of everything upstream, so second
moments add. Each branch input is RMS-normalized, so its second moment is
bounded independent of the stream magnitude; row-stochastic $A$ cannot
increase the maximum value norm (Proposition 2); the $1/(2N)$ output
variance then gives $\mathbb{E}\|b_\ell\|^2 = \Theta(1/N)$, and the sum
over $N$ layers telescopes to $\Theta(1)$.
\end{proof}

\noindent Each hypothesis is load-bearing in the proof; empirically,
however, the gate factor proved dispensable (Table~\ref{tab:ablations}),
so the initialization scaling and normalization carry the result in
practice.

\textbf{Matchings.} Table~\ref{tab:configs} gives the three controls.
Removing the FFN from a 20-layer model (iso-depth) deletes 72\% of its
parameters; matching parameters instead (iso-param) reallocates the budget
into depth, 20 attention-only layers against 4 standard blocks; matching
training FLOPs (iso-FLOP) sits between, because attention pays a
parameter-free quadratic cost in sequence length. No single matching is
the fair one; we report all three.

\begin{table}[t]\centering\footnotesize
\setlength{\tabcolsep}{3.5pt}
\begin{tabular}{@{}lllll@{}}
\toprule
Arm & Config & Total & Non-emb. & GF/tok \\
\midrule
SAN (ours) & 20L $d$512, no FFN & 24.13M & 15.74M & $\sim$0.40 \\
FFN iso-param & 4L $d$512, $f$2048 & 24.12M & 15.73M & $\sim$0.20 \\
FFN iso-FLOP & 9L $d$512, $f$2048 & 43M & 35M & $\sim$0.39 \\
FFN iso-depth & 20L $d$512, $f$2048 & 87.06M & 78.67M & $\sim$0.72 \\
\bottomrule
\end{tabular}
\caption{The three matchings ($f$ = FFN width; GF/tok = training GFLOPs
per token). Removing the FFN perturbs parameters, compute, and depth at
once; each control fixes one axis. Iso-param match is exact to $<$0.04\%.}
\label{tab:configs}
\end{table}

\textbf{Fairness protocol.} Every arm $\times$ optimizer cell received its
own learning-rate sweep at 5B tokens (0.5$\times$/1$\times$/2$\times$
around a base point, with a boundary rule extending any sweep-edge winner
until its curve turned; 17 cells total, Table~\ref{tab:sweep}). The
locked rates are interior minima for all four arms. Global batch (512
rows, 1.05M tokens/step), schedule shape (warmup-stable-decay), data
order, validation set, and all masking are identical across arms. Where a
training phase ran on 7-GPU nodes, batch and learning rates were
frame-matched so effective values equal the 8-GPU locks (511 vs.\ 512
rows, 0.2\%).

\begin{table}[tb]\centering\footnotesize
\setlength{\tabcolsep}{3pt}
\begin{tabular}{@{}lccccccl@{}}
\toprule
Arm & 0.5$\times$ & 1$\times$ & 2$\times$ & 4$\times$ & 8$\times$ &
16$\times$ & Lock \\
\midrule
SAN Muon & 2.248 & \textbf{2.245} & 2.251 & & & & 0.02 \\
FFN Muon & 2.223 & 2.224 & \textbf{2.216} & 2.237 & & & 0.04 \\
SAN AdamW & 2.361 & 2.293 & \textbf{2.267} & 2.321 & & & 6e-4 \\
FFN AdamW & 2.305 & 2.254 & 2.230 & 2.219 & \textbf{2.199} & 2.216 &
2.4e-3 \\
\bottomrule
\end{tabular}
\caption{The learning-rate fairness sweep (val loss at 5B tokens;
1$\times$ is Muon 0.02, AdamW 3e-4). Every lock is an interior minimum.
At the most extreme rate probed, the FFN model's FFN gates collapsed to
mean 0.07 while its attention gates stayed structured: under rate stress
the standard transformer self-pruned toward attention-only form.}
\label{tab:sweep}
\end{table}

\textbf{Noise floor.} Our zero-centered RMSNorm and standard RMSNorm are
optimizer-equivalent under kernel-only weight decay (the gain
parameterizations differ by a constant shift, and Adam-family updates are
shift-invariant). Training both anyway calibrates the same-seed noise
floor: $\Delta = 0.0015$ nats. Every noise comparison below is against
this floor or the three-seed spread. The pipeline's determinism is
strong: a mid-training restore of one run from its 40\% checkpoint, with
cold optimizer state, reproduced the original final validation loss
exactly (1.5982).

\textbf{Evaluation, and two distinct pools.} Headline validation loss
(Table~\ref{tab:main}) is computed on a fixed held-out document pool,
disjoint from training. The region and exercise decomposition
(Table~\ref{tab:regions}) requires exercise metadata that exists only in
the raw corpus, so it uses a second, deterministic 20k-document sample of
the corpus head: approximately 99.9\% of these documents were seen in
training by both arms (about 1.5 times), and documents longer than the
context window are dropped by whole-document packing, a mild selection
toward shorter documents. Decomposition
comparisons are therefore paired contrasts on identical, equally exposed
tokens, valid for localizing the gap but not held-out measurements, and
the two pools' aggregates are not interchangeable (we reconcile them
below). Zero-shot benchmarks use lm-evaluation-harness
\citep{gao2023lmeval}; weight spectra are computed over milestone
checkpoints. Code, training curves, evaluation reports, and all
checkpoints, including quarter-milestones, accompany this work.

\section{The Cost of Removing Feed-Forward Layers}

\begin{table}[t]\centering\footnotesize
\setlength{\tabcolsep}{3.5pt}
\begin{tabular}{@{}llll@{}}
\toprule
Model & Params & Val loss @105B & $\Delta$ vs SAN \\
\midrule
SAN (3 seeds) & 24.13M & $2.0685 \pm 0.0028$ & --- \\
FFN iso-param (3s) & 24.12M & $2.0812 \pm 0.0230$\,$^\dagger$ & $-0.0055$\,$^\ddagger$ \\
FFN iso-FLOP & 43M & 1.8059 & $-0.263$ \\
FFN iso-depth & 87.06M & 1.5982 & $-0.470$ \\
\midrule
\multicolumn{4}{@{}l@{}}{fineweb-edu @31.5B: SAN 3.0131, FFN 2.9733, $\Delta{=}0.0398$} \\
\bottomrule
\end{tabular}
\caption{Deleting FFNs in place (iso-depth) costs 0.47 nats; at matched
FLOPs 0.26; reallocating the budget into attention depth leaves a small,
seed-reproducible 0.006 nats. $^\dagger$Includes one run with a documented terminal
instability (see Discussion); excluding it, $2.0650 \pm 0.0006$.
$^\ddagger$Clean-pair paired $\Delta$: $+0.0055$ (s42), $+0.0054$ (s44);
incident pair $-0.0491$. Same-seed noise floor: 0.0015.}
\label{tab:main}
\end{table}

Table~\ref{tab:main} and Figure~\ref{fig:curves} give the main result as a
monotone sequence. Deleting FFNs in place (iso-depth control, 87M
$\rightarrow$ 24M parameters) costs 0.470 nats. At matched compute
(iso-FLOP), where the FFN arm carries 1.8$\times$ the parameters, it leads
by 0.263. At matched parameters the gap is $+0.0055$ and $+0.0054$ nats on
the two clean seed pairs, agreeing to $10^{-4}$; the third pair contains a
documented terminal-phase instability in the FFN run (gradient-norm ramp
at the learning-rate floor; Table~\ref{tab:main}) and reverses the sign. The
sequence orders exactly by how much parameter budget the control lets
attention reclaim.

\begin{figure}[t]
\centering
\includegraphics[width=\columnwidth]{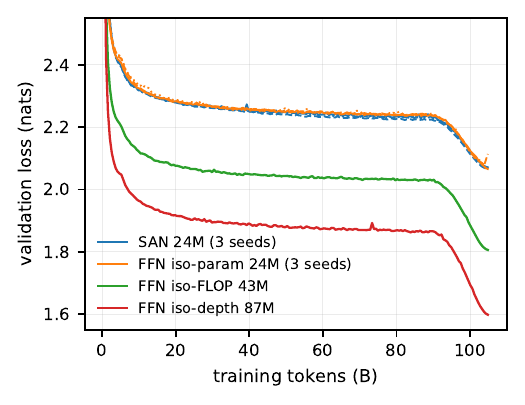}
\caption{Validation loss over 105B tokens: three SAN seeds, three FFN
iso-param seeds, and the iso-FLOP and iso-depth controls. The iso-param
pair is visually indistinguishable through training.}
\label{fig:curves}
\end{figure}

\textbf{The token axis.} The iso-param gap falls from 0.046 nats at 5B
tokens to 0.019 at 30B to 0.0055 at 105B, each budget separately trained
and separately tuned (different schedule horizons), so this is not one
curve read at three points. More training closes the gap.

\textbf{The parameter axis.} At a fixed 31.5B-token budget across five
matched size pairs (Figure~\ref{fig:ladder}), the gap is $-0.045$ nats at
the smallest size (the 2-layer FFN partner hits a depth floor first),
crosses zero, and plateaus at $\approx$0.02 nats from 16M to 57M
non-embedding parameters. The gap does not grow with scale in the tested
range.

\begin{figure}[t]
\centering
\includegraphics[width=\columnwidth]{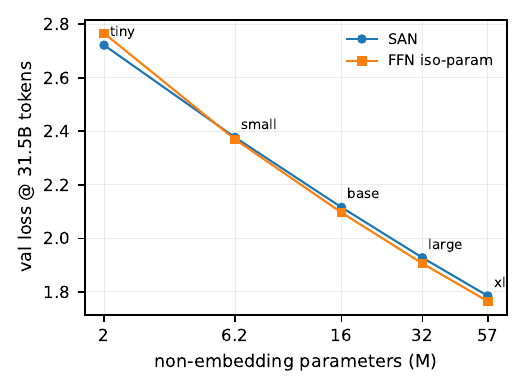}
\caption{Iso-token scaling (31.5B tokens per size). The SAN wins at the
smallest size; from 16M non-embedding parameters upward the gap is flat at
$\approx$0.02 nats.}
\label{fig:ladder}
\end{figure}

\textbf{Repetition.} Constraining both arms to 2M, 8M, or 32M unique
documents at fixed tokens (up to 18 epochs) costs at most 0.010 nats and
shows no architecture $\times$ repetition interaction
(Figure~\ref{fig:repetition}), extending data-constrained scaling results
\citep{muennighoff2023scaling} to synthetic data.

\begin{figure}[tb]
\centering
\includegraphics[width=0.92\columnwidth]{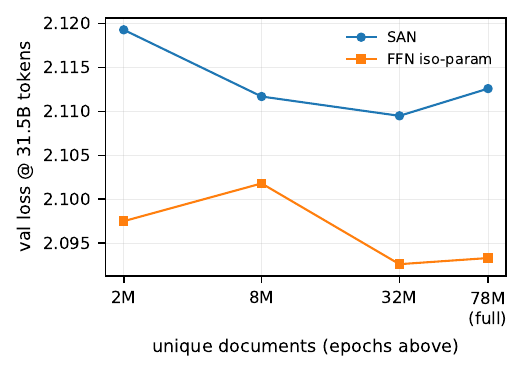}
\caption{Unique-document constraints at a fixed 31.5B-token budget. Both
arms are nearly repetition-insensitive, with no architecture
interaction.}
\label{fig:repetition}
\end{figure}

\section{Where the Gap Lives}

\begin{table}[tb]\centering\footnotesize
\setlength{\tabcolsep}{4pt}
\begin{tabular}{@{}lcccc@{}}
\toprule
& query & trace & answer & all \\
\midrule
token share (\%) & 5.8 & 57.2 & 35.9 & 100 \\
loss share (\%) & 7.9 & 55.4 & 36.6 & 100 \\
\midrule
$\Delta$ @31B & $+.052$ & $+.008$ & $+.011$ & $+.011$ \\
$\Delta$ @105B (clean) & $+.038$ & $-.004$ & $-.007$ & $-.0025$ \\
\bottomrule
\end{tabular}
\caption{Token-weighted decomposition on the decomposition sample (Setup;
training-exposed, not the held-out pool): $\sum$ share $\times$
region-$\Delta$ reproduces the sample aggregate to within 2\%. Point
estimates; per-block losses were not retained. The deficit
concentrates on low-context query tokens and localizes there entirely by
105B; the same account predicts the fineweb gap (measured 0.040,
pre-registered window 0.02--0.05).}
\label{tab:regions}
\end{table}

\textbf{Token regions.} Every corpus document has query, reasoning-trace,
and answer regions delimited by atomic marker tokens.
Table~\ref{tab:regions} and Figure~\ref{fig:decomp} decompose the
iso-param gap on the decomposition sample. At 31B tokens the per-token
deficit on query regions ($+0.052$) is five times the sample aggregate
while carrying 8\% of its loss; the token-weighted sum of region gaps
reproduces the sample aggregate to within 2\%, so the decomposition is
exhaustive. By 105B the localization is complete on this sample: the SAN
leads on every answer region, including memorization exercises, and on
traces; the deficit survives only on query tokens ($+0.038$), the
positions with the least context to route from. Note the pool distinction
of the Setup: the decomposition sample's 105B aggregate ($-0.0025$, SAN
ahead) differs from the held-out validation gap ($+0.0055$, FFN ahead)
because the samples differ in composition and training exposure; both are
parity-scale, the within-sample paired contrasts that localize the gap
are unaffected, and the two aggregates are not interchangeable. Across
the size ladder the query deficit is the invariant, positive at every
size and budget while every other region changes sign
(Figure~\ref{fig:decomp}b). The query-localized account is itself a
revision: the registered prediction (gap largest on memorization answers)
was falsified by these measurements, and only the revised account's
fineweb consequence below was pre-registered.

\begin{figure*}[t]
\centering
\includegraphics[width=\textwidth]{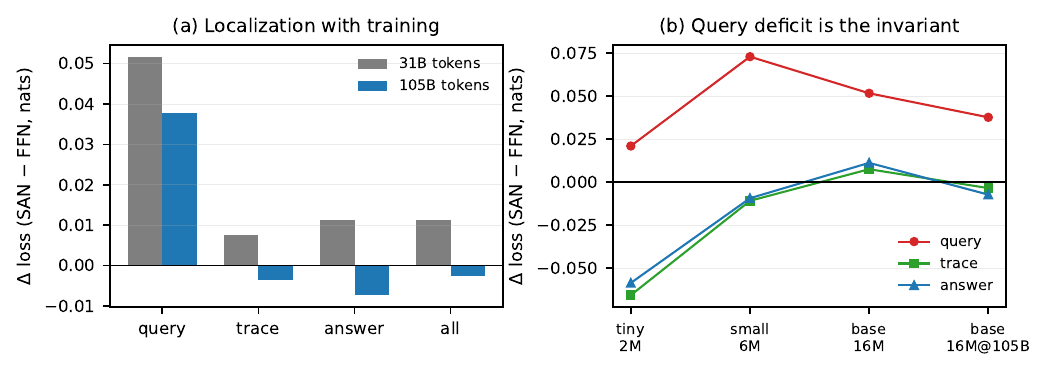}
\caption{(a) Region-level gap at 31B vs.\ 105B tokens: with training the
deficit localizes entirely to low-context query prediction. (b) Across
matched size pairs, the query deficit is the only invariantly FFN-favored
quantity.}
\label{fig:decomp}
\end{figure*}

\textbf{Task types.} On answer regions by exercise, the sign flips in the
storage-account's direction: at 31B the SAN is already better on rag,
editing, and math answers (retrievable from prompt or trace) and worse on
memorization and free-form generation; by 105B it leads on all of them.

\begin{table*}[t]\centering\footnotesize
\begin{tabular}{lccccccc}
\toprule
Model & lambada (ppl) & sciq & arc\_easy & piqa & hellaswag & winogrande & mmlu \\
\midrule
chance & $\sim$0 & 0.25 & 0.25 & 0.50 & 0.25 & 0.50 & 0.25 \\
\midrule
SAN 24M @31B & 0.091 (821) & 0.725 & 0.397 & 0.545 & 0.279 & 0.504 & 0.236 \\
FFN iso-param @31B & 0.136 (492) & 0.702 & 0.416 & 0.563 & 0.283 & 0.516 & 0.233 \\
SAN 24M @105B (3s) & 0.111 (682) & 0.742 & 0.405 & 0.559 & 0.284 & 0.517 & 0.232 \\
FFN iso-param @105B (3s) & 0.134 (518) & 0.661 & 0.415 & 0.558 & 0.283 & 0.511 & 0.245 \\
FFN iso-FLOP 43M @105B & 0.178 (451) & 0.634 & 0.419 & 0.562 & 0.293 & 0.534 & 0.230 \\
FFN iso-depth 87M @105B & 0.111 (964) & 0.630 & 0.426 & 0.552 & 0.288 & 0.525 & 0.233 \\
\midrule
SAN fineweb-trained @31B & 0.203 (189) & 0.705 & 0.447 & 0.579 & 0.294 & 0.507 & 0.231 \\
FFN fineweb-trained @31B & 0.181 (218) & 0.707 & 0.478 & 0.588 & 0.287 & 0.507 & 0.232 \\
\bottomrule
\end{tabular}
\caption{The storage/routing split at task level. Lambada (out-of-distribution
recall for SYNTH-trained models) favors the FFN at every budget, but is
non-monotone in FFN capacity: iso-depth (87M), the best in-distribution model,
is the worst at it (0.111, ppl 964). Sciq (support passage in context) favors
the SAN with a margin that grows with training in the pre-registered direction
(seed ranges non-overlapping at 105B) and falls monotonically with FFN
capacity. Trained on knowledge-dense fineweb-edu, the lambada preference
reverses: a task's storage/routing identity is relative to the training
distribution. Remaining tasks are at chance at these scales.}
\label{tab:downstream}
\end{table*}

\textbf{Benchmarks.} Table~\ref{tab:downstream} shows the same split at
task level. Lambada, predicting a specific content word and, for
SYNTH-trained models, an out-of-distribution recall task, favors the FFN
arm at every budget. Sciq, whose answer sits in a provided support
passage, favors the SAN, and the margin grows with training in a
direction we registered in advance: from 31B to 105B the SAN improves
(0.725 $\rightarrow$ 0.742) while the FFN model regresses (0.702
$\rightarrow$ 0.661), with non-overlapping seed ranges. The remaining
tasks sit at chance at these scales.

\textbf{A pre-registered out-of-distribution test.} The account so far
says the FFN's advantage is low-context prediction. Natural web text is
mostly low-context prediction. Before launching the run, we registered
the prediction that an iso-param pair trained on fineweb-edu
\citep{penedo2024fineweb} would show a gap of 0.02 to 0.05 nats, bounded above by the pure query-region deficit
and well above the SYNTH aggregate. The measured gap is 0.0398.
Instructively, the same pair reverses on lambada (SAN 0.203 vs.\ FFN
0.181): trained on distribution-matched text, the passage suffices to
infer the final word, and lambada becomes a routing task. A task's
storage-versus-routing identity is relative to the match between training
distribution and task, a point the discussion returns to.

\section{Mechanism: Weight-Spectrum Dynamics}

\begin{figure*}[t]
\centering
\includegraphics[width=\textwidth]{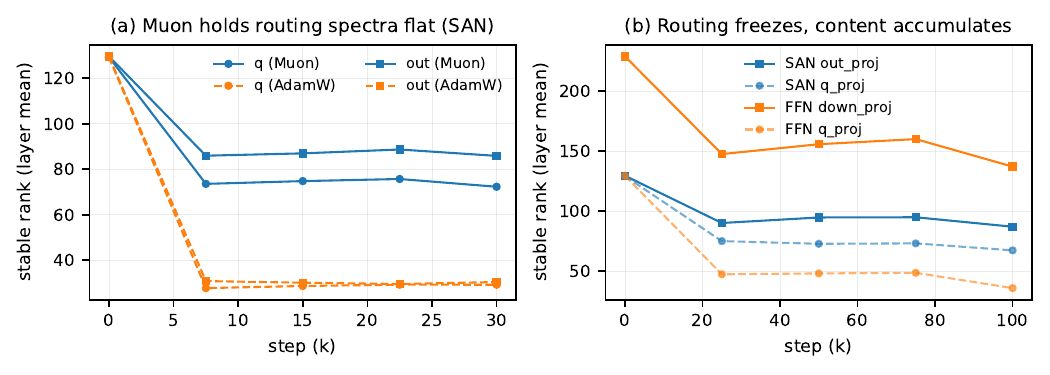}
\caption{Stable rank ($\|W\|_F^2/\|W\|_2^2$, layer mean) of weight
matrices over training. (a) Muon holds Q and output-projection spectra
2--3$\times$ flatter than AdamW in the SAN. (b) Routing matrices (Q)
crystallize by a quarter of training in every model; write-path matrices
(SAN $W_o$, FFN down-projection) accumulate rank through the stable phase,
contracting only in the learning-rate decay tail.}
\label{fig:spectra}
\end{figure*}

Figure~\ref{fig:spectra} tracks the stable rank of every weight matrix
over milestone checkpoints. Two regularities hold across all
architectures, sizes, and budgets we trained. First, routing matrices
(Q/K) spectrally crystallize within the first quarter of training and do
not move thereafter; the optimizer sets the level (Muon holds them
2--3$\times$ flatter than AdamW, alike in both architectures; the value
projection is optimizer-invariant), but the schedule is universal. The Muon side is the
a-priori expectation: for orthogonalized updates, Weyl's inequality bounds
every singular value's per-step drift by the learning rate, though the
bound controls drift only and does not by itself predict the level
separation we observe. Second, the matrices
that write content into the residual stream accumulate stable rank for as
long as the stable phase lasts: the FFN down-projection in FFN models, and
in attention-only models the output projection $W_o$, which inherits the
role. This is the training-dynamics face of the QK/OV
routing-versus-content decomposition \citep{elhage2021mathematical}:
selection structure is learned early and frozen; content capacity
accrues indefinitely, in whichever matrices can hold it. Representation
effective rank stays high everywhere (minimum layer rank 173 of 512 at 20
layers), so the classical collapse regime
\citep{dong2021attention,noci2022signal} is never approached with
residuals and normalization in place.

\section{Component Ablations}

\begin{table}[tb]\centering\footnotesize
\setlength{\tabcolsep}{3.5pt}
\begin{tabular}{@{}lll@{}}
\toprule
Variant & Val loss & $\Delta$ vs gated \\
\midrule
gated (baseline, 20L) & 2.1343 & --- \\
$+$ sandwich norm & \textbf{2.1251} & $-0.0092$ \\
standard residual & 2.1330 & $-0.0013$ \\
ReZero residual & 2.1375 & $+0.0032$ \\
RMSNorm $\gamma{=}1$ (control) & 2.1358 & $+0.0015$ \\
no residual & 2.8826 & $+0.75$ \\
no QK-norm & \emph{diverges} (8.28) & --- \\
\midrule
depth 8/32/48L, gated & 2.168/2.153/2.175 & U-shaped \\
\quad standard residual & 2.164/2.150/2.177 & gates neutral \\
FFN 20L gated/standard & 1.6977/\textbf{1.6930} & neutral \\
\bottomrule
\end{tabular}
\caption{QK-norm is the load-bearing component; gates are
performance-neutral at every depth and in both architectures (their value
was diagnostic); sandwich norm is the only improvement. The RMSNorm
control is optimizer-equivalent by construction and calibrates the
same-seed noise floor (0.0015).}
\label{tab:ablations}
\end{table}

Table~\ref{tab:ablations} isolates the components. Three results matter.
First, QK-normalization is load-bearing: removing it diverges outright at
the tuned learning rate, the only divergence in the study and a finding
none of our predictions anticipated; the claim is scoped to the tuned
rate. Second, residual gating is performance-neutral
everywhere, at 20--48 layers in the SAN arm and in the FFN mirror, matching
ReZero-family expectations \citep{bachlechner2020rezero}; the gates'
value was diagnostic, as their trajectories exposed the FFN arm's
self-pruning under learning-rate stress (Table~\ref{tab:sweep}). Third, the only variant to beat the baseline is sandwich
normalization \citep{shleifer2021normformer} at $-0.009$ nats. Depth at
iso-param is U-shaped with a 20-layer optimum, and 48-layer attention-only
stacks train without incident. The optimizer interaction is
budget-dependent: a full crossover at 5B tokens (each architecture
preferring a different optimizer, Table~\ref{tab:sweep}) washes out
entirely by 30B (Table~\ref{tab:opt}).

\begin{table}[tb]\centering\footnotesize
\setlength{\tabcolsep}{5pt}
\begin{tabular}{@{}lcc@{}}
\toprule
Val loss @30B & Muon & AdamW \\
\midrule
SAN 24M & 2.1126 & 2.1103 \\
FFN iso-param 24M & 2.0933 & 2.0950 \\
\bottomrule
\end{tabular}
\caption{The optimizer $\times$ architecture grid at 30B tokens, each
cell at its own tuned rate: within-architecture deltas are $\le 0.002$
nats, against a 5B-token crossover of 0.022 and 0.017
(Table~\ref{tab:sweep}). The interaction is a short-horizon phenomenon.}
\label{tab:opt}
\end{table}

\textbf{A recipe, in one paragraph.} For attention-only stacks at these
scales: QK-normalization is mandatory (its removal is the study's only
divergence); post-attention sandwich normalization is a free improvement;
scalar residual gates cost nothing and expose useful training
diagnostics, but standard residuals match them at every depth tested;
depth prefers 20 layers at this width, and 48 layers remain trainable.
More consequential than any of these, tune the learning rate per
architecture: the optimal Muon rate differs by 2$\times$ between the
arms, and a shared rate silently biases any comparison.

\section{Discussion and Limitations}

\textbf{What the FFN buys.} Under controls, the transformer's FFN is not
a distinct computational primitive whose removal breaks the model; it is
parameter capacity attached to a write path. Give the same capacity to
attention depth and, on reasoning-dense data, what remains is a small,
highly reproducible cost: 0.006 nats, 0.27\% of loss, agreeing to
$10^{-4}$ across clean seed pairs, well outside the paired-seed spread
and 3.7$\times$ the same-seed floor, reversing sign only under a
documented FFN-arm instability. We claim a tiny, real effect, not a null.
What is lost specifically is performance on tokens where context provides
nothing to route, and the loss relocates, in weight space, to whichever
matrices write content into the stream. Separately, every instability
event in the study occurred in the FFN arm and none in the SAN arm
(gate self-pruning under learning-rate stress, Table~\ref{tab:sweep};
one terminal-phase divergence, Table~\ref{tab:main}), a
practitioner-relevant robustness asymmetry
we report as observation.

\textbf{Distribution-relative task identity.} The fineweb reversal on
lambada and the iso-depth model's pattern (best in-distribution loss,
worst out-of-distribution recall) argue that storage-versus-routing is
not a fixed property of a task: it is a property of the match between the
task and the training distribution. This reframes small-model benchmark
comparisons generally, and it cautions against reading our SYNTH results
as distribution-free.

\textbf{When would one use a SAN?} At matched parameters the SAN pays
$\approx$2$\times$ FLOPs per token at 2048 context; iso-FLOP favors the
FFN. The trade favors attention-only models where parameters, not FLOPs,
bind (on-device memory limits), where architectural simplicity has
analysis value (a SAN is one repeated, interpretable primitive), and on
trace-rich distributions. We claim regime, not superiority.

\textbf{Limitations.} All results are at or below 87M total parameters
and 105B tokens, on one reasoning-dense corpus plus one knowledge-dense
control pair; MMLU-class benchmarks are at chance throughout and cannot
yet discriminate. Size-flatness is measured at 31.5B tokens and token
convergence at the base size only; their conjunction at larger scale is
extrapolation. The region decomposition reports point estimates
(per-block intervals were not retained), and its confirmation of the
revised account rests on a single out-of-distribution point. The storage account itself predicts a wider gap on
storage-heavy mixtures at larger scale, which is the natural next
experiment. Of the eight predictions fixed before their measurements, two
were falsified outright; the central localization finding emerged from
one of those failures.

\bibliography{references}

\section{Appendix A: Additional Formal Results}

Propositions 1--2, the scalar-bottleneck lemma, and the bounded-variance
theorem are stated and proved in the body of the paper. Two further results
inform component choices and the optimizer analysis. Notation as in the
main text: norm gain $(1+\gamma)$ with $\gamma$ init 0.

\begin{proposition}[Weight-decay fixed point of norm gains]
Under $L_2$ decay, the penalty on $\gamma$ in the $(1+\gamma)$
parameterization is minimized at gain 1 (identity); under the standard
parameterization (gain $h$, init 1) it is minimized at $h=0$, the zero
map.
\end{proposition}
\begin{proof}
Both objectives are strictly convex with the stated minimizers.
\end{proof}
In the study's fixed frame, weight decay applies to kernels only, so the
two parameterizations are additionally \emph{optimizer-equivalent}: they
differ by a constant shift, and Adam-family updates are shift-invariant.
Training both anyway yielded $\Delta = 0.0015$ nats, the same-seed noise
floor used throughout. The proposition's content, that zero-centering
makes ``do nothing'' the regularizer's fixed point, becomes operative
only in frames that decay norm gains.

\begin{lemma}[Muon bounds per-step spectral drift, and only that]
Let the update be $\Delta = -\eta P$ with $P$ the polar factor of the
momentum gradient (Newton--Schulz approximation error $\varepsilon$). By
Weyl's inequality for singular values, for every $i$,
$|\sigma_i(W+\Delta)-\sigma_i(W)| \le \|\Delta\|_2 = \eta(1+\varepsilon)$;
after $t$ steps every singular value lies within $t\eta$ of its
initialization.
\end{lemma}
This bounds drift; it does not prove spectra stay balanced nor that
balance improves loss. The measured outcome is that Muon holds
routing-matrix spectra 2--3$\times$ flatter than AdamW in both
architectures alike: the effect is real but not architecture-differential.

\section{Appendix B: Registration Record}

All hypotheses and predictions were fixed before the corresponding
measurements; the fineweb-edu prediction was fixed numerically before
that run launched. Outcomes against the registered predictions:

\begin{table}[h]\centering\small
\begin{tabular}{p{0.12\columnwidth}p{0.78\columnwidth}}
\toprule
P1 & \textbf{Confirmed.} Iso-param gap within the seed band at 105B
(clean pairs $+0.0055/+0.0054$). \\
P2 & \textbf{Confirmed at 31B} (trace gap smallest); at 105B trace and
answer both reach parity, superseded by full localization. \\
P3 & \textbf{Falsified as stated.} The gap peaks on low-context query
tokens, not memorization answers; the restated account then made the
confirmed fineweb prediction. \\
P4 & \textbf{Gates component falsified} by its own criterion (gated
$\approx$ ungated everywhere); trainability confirmed; QK-norm emerged,
unpredicted, as the load-bearing component. \\
P5 & \textbf{Not exercised}: the fixed frame decays kernels only; the
measurement instead calibrated the noise floor (0.0015). \\
P6 & \textbf{Budget-dependent}: full optimizer$\times$architecture
crossover at 5B; gone by 30B. \\
P7 & \textbf{Half-confirmed}: spectra flatter under Muon everywhere, but
not differentially in the SAN; representation rank high everywhere
(minimum 173/512 at 20 layers). \\
P8 & \textbf{Falsified}: no architecture$\times$repetition interaction;
both arms nearly repetition-insensitive to 18 epochs. \\
\bottomrule
\end{tabular}
\caption{Registered predictions vs.\ outcomes. The instruments were
designed so that predictions could fail; two did, and one failure
produced the paper's central finding.}
\end{table}

\section{Appendix C: Learning-Rate Probes for the Size Ladder}

The size ladder reuses the base-tuned rates at every size. Before
interpreting its curve shape, 3-point rate probes
(0.5$\times$/1$\times$/2$\times$ of the lock, 2.5k-step horizon, since a
30k-horizon run's early validation is not schedule-comparable) at the
smallest and largest sizes confirmed the transferred lock is in the flat
region at tiny (2.851/2.856/2.857) and an interior minimum at xl
(2.074/2.066/2.078); no per-size sweeps were needed.

\section{Appendix D: Per-Seed Downstream Results}

\begin{table*}[t]\centering\footnotesize
\begin{tabular}{lccccccc}
\toprule
Model & lambada (ppl) & sciq & arc\_easy & piqa & hellaswag & winogrande & mmlu \\
\midrule
SAN @105B, seed 42 & 0.109 (733) & 0.718 & 0.384 & 0.562 & 0.284 & 0.518 & 0.231 \\
SAN @105B, seed 43 & 0.115 (633) & 0.752 & 0.414 & 0.560 & 0.285 & 0.530 & 0.232 \\
SAN @105B, seed 44 & 0.110 (682) & 0.757 & 0.419 & 0.555 & 0.282 & 0.504 & 0.232 \\
\midrule
FFN iso-param @105B, seed 42 & 0.143 (577) & 0.650 & 0.401 & 0.559 & 0.287 & 0.494 & 0.234 \\
FFN iso-param @105B, seed 43$^\dagger$ & 0.128 (482) & 0.660 & 0.426 & 0.559 & 0.279 & 0.518 & 0.247 \\
FFN iso-param @105B, seed 44 & 0.132 (493) & 0.673 & 0.418 & 0.555 & 0.282 & 0.521 & 0.254 \\
\bottomrule
\end{tabular}
\caption{Per-seed 0-shot results behind the 105B means of the main text.
Sciq seed ranges do not overlap between arms (SAN 0.718--0.757 vs.\ FFN
0.650--0.673). $^\dagger$The seed with the terminal-phase instability.}
\label{tab:downstream_seeds}
\end{table*}

Table~\ref{tab:downstream_seeds} gives the per-seed values behind the
105B means reported in the body of the paper. The capacity-axis rows (iso-FLOP,
iso-depth) and the fineweb-trained pair appear in the body of the paper's
downstream table.

\section{Appendix E: Repetition: Exact Values}

Validation loss at 30k steps (31.5B tokens) under unique-document
constraints:

\begin{table}[h]\centering\small
\begin{tabular}{lcccc}
\toprule
Unique docs & 2M ($\sim$18 ep.) & 8M & 32M & 78M (full) \\
\midrule
SAN & 2.1193 & 2.1117 & 2.1095 & 2.1126 \\
FFN iso-param & 2.0975 & 2.1018 & 2.0926 & 2.0933 \\
\bottomrule
\end{tabular}
\caption{Both arms are nearly repetition-insensitive; within-arm
non-monotonicity ($\sim$0.005--0.009) is comparable to data-subset
variance, and there is no architecture interaction.}
\end{table}

\section{Appendix F: Gate Trajectories and FFN-Arm Observations}

\begin{figure}[h]
\centering
\includegraphics[width=0.9\columnwidth]{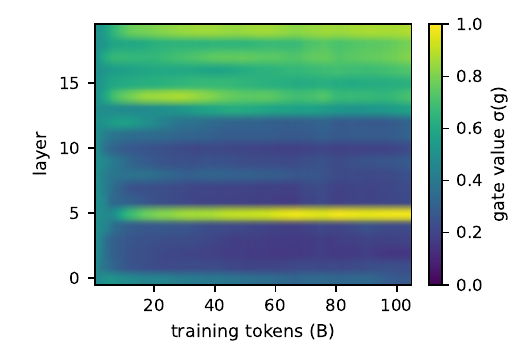}
\caption{Gate values over training, SAN 105B run (layers $\times$ time).
Early layers self-attenuate and late layers sharpen under Muon.}
\end{figure}

Although performance-neutral, the gates were diagnostically productive.
Three independent observations of FFN-arm fragility: (i) at 16$\times$ the
tuned rate, all FFN gates collapsed (mean 0.07) while attention gates
stayed structured; (ii) at the tuned rate under Muon, FFN gates ended low
(mean 0.16--0.17) in three separate runs, echoing the same self-pruning;
(iii) one of three FFN seeds at 105B suffered a terminal-phase
instability, its gradient norm tripling over the final 800 steps at the
learning-rate floor (Figure~\ref{fig:tail}); no SAN run showed any of
these behaviors. We report these as observations, not claims.

\begin{figure}[h]
\centering
\includegraphics[width=0.9\columnwidth]{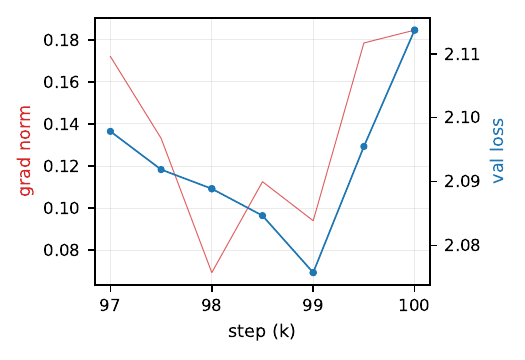}
\caption{The FFN iso-param seed-43 terminal instability. The run tracked
its sibling seed to within 0.002 nats through step 99{,}000. Its final
value is included, flagged, in all statistics.}
\label{fig:tail}
\end{figure}

\section{Appendix G: Depth Ladder}

\begin{figure}[h]
\centering
\includegraphics[width=0.9\columnwidth]{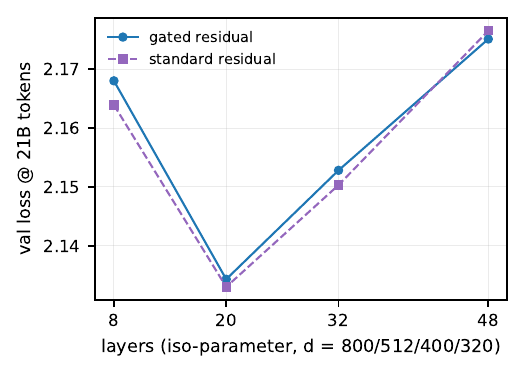}
\caption{Iso-parameter depth ladder for the SAN arm, gated vs.\ standard
residuals: U-shaped in depth with a 20-layer optimum; gates neutral at
every depth; 48 layers train without incident.}
\end{figure}

\section{Appendix H: Reproducibility Details}

\textbf{Infrastructure.} Training ran on 8$\times$H100 80GB nodes
(two phases on 7$\times$H100, frame-matched below) with single-node data
parallelism, bfloat16 activations, JAX 0.10 with CUDA 12.9 on Ubuntu
22.04; evaluation on a single GPU.

\textbf{Frame matching.} All headline results use a fixed global batch of
512 rows (64/device $\times$ 8) of 2048 tokens. Phases run on 7-GPU nodes
used 73 rows/device (511 total, $-0.2\%$) with learning rates rescaled so
effective values equal the 8-device locks (Adam scales linearly in device
count, Muon as its square root). A base-size pair rerun across frames
agreed to 0.0033 nats.

\textbf{Determinism and recovery.} Checkpoints store parameters, not
optimizer state. One run was resumed from its 40\%-checkpoint after an
artifact loss; despite cold optimizer state at the stable-phase rate, it
reproduced the original final validation loss exactly (1.5982),
suggesting the warm-restart transient is fully absorbed by the
learning-rate decay phase.

\textbf{Statistics.} Headline comparisons use three seeds per arm with
paired per-seed differences; the same-seed noise floor (0.0015 nats) comes
from the optimizer-equivalent normalization pair. Per-block losses were
not retained, so block-level bootstrap intervals are not reported;
seed-level agreement (clean pairs matching to $10^{-4}$) is the operative
evidence.

\textbf{Artifacts.} Code, tokenizer, corpus tokenization manifests, all
checkpoints with 25\% milestones, training logs, and every evaluation
report accompany this work.

\end{document}